\documentclass[11pt]{article}
\usepackage{amsmath}[11pt]
\usepackage{amsmath, amsthm, amssymb}
\usepackage{graphics}
\usepackage{amsthm}
\usepackage{amsfonts}[11pt]
\usepackage{graphicx}
\usepackage{enumerate}
\usepackage{algorithm}
\usepackage{algorithmic}

\newtheorem{thm}{Theorem}
\newtheorem{cor}[thm]{Corollary}
\newtheorem{lem}[thm]{Lemma}
\newtheorem{proposition}[thm]{Proposition}
\newtheorem*{propositionOCO}{Proposition (OGD regret)}
\newcommand{\bc}{\begin{center}}
\newcommand{\ec}{\end{center}}
\newcommand{\bq}{\begin{quote}}
\newcommand{\eq}{\end{quote}}

\newcommand{\be}{\begin{equation}}
\newcommand{\ee}{\end{equation}}
\newcommand{\beqa}{\begin{eqnarray*}}
\newcommand{\eeqa}{\end{eqnarray*}}
\newcommand{\beqn}{\begin{eqnarray}}
\newcommand{\eeqn}{\end{eqnarray}}
\newcommand{\bbibl}{\begin{thebibliography}{99}}
\newcommand{\ebibl}{\end{thebibliography}}

\newcommand{\ba}{\begin{array}}
\newcommand{\ea}{\end{array}}

\DeclareMathOperator*{\argmax}{argmax}
\DeclareMathOperator*{\argmin}{argmin}

\newcommand\reals{\mathbb{R}}
\newcommand\E[1]{\mathbb{E}\left[#1\right]}

\begin{document}

\title{Perceptron-like Algorithms and Generalization Bounds for Learning to Rank}
\author{Sougata Chaudhuri and Ambuj Tewari \\ \texttt{\{sougata,tewaria\}@umich.edu}}



\maketitle



\begin{abstract}
Learning to rank is a supervised learning problem where the output space is the space of rankings but the supervision space is the space of relevance scores. We make theoretical contributions to the learning to rank problem both in the online and batch settings. First, we propose a \emph{perceptron}-like algorithm for learning a ranking function in an online setting. Our algorithm is an extension of the classic perceptron algorithm for the classification problem. Second, in the setting of batch learning, we introduce a \emph{sufficient condition} for convex ranking surrogates to ensure a generalization bound that is independent of number of objects per query. Our bound holds when linear ranking functions are used: a common practice in many learning to rank algorithms. En route to developing the online algorithm and generalization bound, we  propose a novel family of \emph{listwise} large margin ranking surrogates. Our novel surrogate family is obtained by modifying a well-known \emph{pairwise} large margin ranking surrogate and is distinct from the listwise large margin surrogates developed using the structured prediction framework. Using the proposed family, we provide a guaranteed upper bound on the cumulative NDCG (or MAP) induced loss under the perceptron-like algorithm. We also show that the novel surrogates satisfy the generalization bound condition.
\end{abstract}

\section{Introduction}
\label{introduction}

Learning to rank is a supervised learning problem where the output space  is the space of \emph{rankings} of a set of objects. In the learning to rank problem that frequently arises in information retrieval, the objective is to rank \emph{documents} associated with a \emph{query}, in the order of the \emph{relevance} of the documents for the given query. During training, a number of queries, each with their associated documents and relevance levels, are provided. A ranking function is learnt by using the training data with the hope that it will accurately order documents for a test query, according to their respective relevance levels. In order to measure the accuracy of a ranked list, in comparison to the actual relevance scores, various ranking performance measures, such as NDCG \cite{jarvelin2002}, MAP \cite{baeza1999} and others, have been suggested.

All major performance measures are non-convex and discontinuous in the scores. Therefore, optimizing them during the training phase is a computationally difficult problem. For this reason, several existing ranking methods are based on minimizing \emph{surrogate} losses, which are easy to optimize. Ranking methods can be broadly categorized into three categories. In the \emph{pointwise} approach, the problem is formulated as regression or classification problem, with the objective of predicting the true relevance level of individual documents \cite{Cossock2006}. In the \emph{pairwise} approach, document pairs are taken as instances, and the problem is reduced to binary classification (which document in a pair is more relevant?). Examples include RankSVM \cite{herbrich1999}, RankBoost \cite{Freund2003}, and RankNet \cite{Burges2005}. In the \emph{listwise} approach, the entire list of document associated with a query is taken as an instance, and listwise surrogates are minimized during training. Examples include ListNet \cite{Cao2007} and AdaRank \cite{Xu2007}.

The listwise method for ranking has become popular since the major performance measures themselves are listwise in nature. Usually, listwise surrogates are used in conjuction with linear ranking functions so that powerful optimization algorithms can be used. Despite the plethora of existing ranking methods, the comparison between them is mainly based on empirical performance on a limited set of publicly available data sets. Moreover, it has been observed that non-linear ranking function, in conjunction with even simple surrogates, are hard to beat in practice \cite{chapelle2011yahoo}. Important theoretical questions, such as online algorithms with provable guarantees and batch algorithms with generalization error bounds, remain open \cite{chapelle2011}, even for linear ranking functions.

Listwise large margin surrogates form an important sub-class of listwise surrogates. Their use is motivated by the success of large margin surrogates in supervised classification problems. However, existing popular listwise large margin surrogates in the learning to rank literature are derived using the structured prediction framework \cite{chapelle2007, yue2007, chakrabarti2008}. In standard structured prediction, the supervision space is the same as the output space of the function being learned. To fit the structured prediction framework to the learning to rank problem (where the supervision is in form of relevance vectors but the output space consists of \emph{full rankings} of the documents associated with a query), the relevance vectors are \emph{arbitrarily} mapped to full rankings. Though such an approach can yield good empirical results, it does not lead to well-defined surrogates in the learning to rank setting since the mapping from relevance scores to full rankings is left unspecified (or is arbitrarily chosen). 

One important reason for investigating listwise large margin ranking surrogate is to develop an analogue of the \emph{perceptron} algorithm used in classification \cite{freund1999}. In classification, large margin surrogates have been used to learn classifiers in an online setting using perceptron. Large margin surrogates have special properties that allow for the establishment of theoretical bounds on the cumulative zero-one loss (viz. the total number of \emph{mistakes}) without making any statistical assumptions on the data. Perceptron-like algorithms have been developed for ranking
but in a different setting \cite{crammer2001}. To the best of our knowledge, the perceptron algorithm has not been extended to the learning to rank setting described in this paper where, instead of mistake bounds, we desire bounds on cumulative losses as measured by the popular listwise ranking measures such as NDCG and MAP. 

The three main contributions of this paper are the following.
First, we modify a popular \emph{pairwise} large margin ranking surrogate to develop a family of \emph{listwise} large margin ranking surrogates. The family is parameterized by a set of weight vectors that gives us the flexibility to upper bound losses induced by NDCG and MAP. Unlike
surrogates designed from a structured prediction perspective, ours directly use the relevance scores and do not require an arbitrary map from relevance scores to full rankings.
Second, we use the novel family of surrogates to develop a perceptron-like algorithm for learning to rank. We provide theoretical bounds on the
cumulative NDCG and MAP induced losses. If there is a perfect linear ranking function which can rank every instance correctly, the loss bound is independent of number of training instances just as in the classic perceptron case.
Third, we analyze the generalization bound of the proposed family to understand its performance in a batch setting. In doing so, we provide a sufficient condition for \emph{any} ranking surrogate (with linear ranking functions) to have a generalization bound independent of number of documents per query. We show that the proposed family and few other popular ranking surrogates satisfy the sufficient condition.

We defer all proofs to the supplementary appendix.

\section{Problem Definition}
\label{probdef}
In learning to rank, an instance consist of a query $q$, associated with a list of $m$ documents and corresponding relevance label vector of length $m$. The documents are represented as $d$ dimensional feature vectors. The relevance labels represent how relevant the documents are to the query. The relevance vector can be binary or multi-graded (say $0$ through $4$). Formally, the input space is $\mathcal{X} \in \mathbb{R}^{m \times d}$ representing lists of $m$ documents represented as $d$ dimensional feature vectors and supervision space is $\mathcal{Y} \in \mathbb{R}^m$, representing relevance label vectors. 
It is important to note that the supervision \emph{is not in the form of full rankings}. In fact, a list of documents usually has multiple correct full rankings corresponding to the relevance vector.

The objective is to learn a ranking function which ranks the documents associated with a query. The prevalent technique in the literature is to learn a scoring function and get ranking by sorting the score vector. For a $X \in \mathcal{X}$, a linear scoring function is $f_{w}(X)= Xw= s^w \in \mathbb{R}^m$, where $w \in \mathbb{R}^d$. The quality of the learnt ranking function is evaluated on an independent test query by comparing the ranks of the documents according to the scores, and their ranks according to actual relevance labels, using various performance measures. For example, the Normalized Cumulative Discounted Gain (NDCG) measure, for a set of $m$ documents in a test query, with multi-graded relevance vector $R$ and score vector $s$ induced by ranking function, is defined as follows:
\begin{equation}
\label{eq:NDCG}
\begin{split}
NDCG(s,R)= \frac{1}{Z(R)}\sum_{i=1}^m G(R_i)D(\pi^{-1}_s(i)) 
\end{split}
\end{equation}
where $G(r)= 2^r -1$, $D(i)= \frac{1}{\log_2{(i+1)}}$, $Z(R)= \underset {\pi}{\max}\sum_{i=1}^m G(R_i)D(\pi^{-1}(i))$. Further, $R_i$ is the relevance level of document $i$ and $\pi^{-1}_s(i)$ is the rank of document $i$ in the permutation $\pi_s$ ($\pi_s$ is the permutation induced by score vector $s$). For example, if document $1$ is placed 3rd in permutation $\pi_s$, then $\pi^{-1}_s(1)=3$.  

Another popular performance measure, Mean Average Precision (MAP), is defined only for binary relevances:
\begin{equation}
\label{eq:MAP}
\begin{aligned}
MAP(s,R) = \frac{1}{r} \sum_{j:R(\pi_s(j))=1} \frac{\sum_{i \le j} I[R(\pi_s(i))=1]}{j}
\end{aligned}
\end{equation}
where $r$ is the total number of relevant documents in the set of $m$ documents. Note that $\pi_s(j)$ indicates the document which is placed at position $j$ in permutation $\pi_s$. Thus, if $\pi_s(3)=1$, that means the document in 3rd position in $\pi_s$ is document 1.

All ranking performances measures are actually \emph{gains} intended to be maximized. When we say ``NDCG induced loss", we mean a loss function that simply subtracts NDCG from its maximum possible value (which is $1$). Similar losses can be induced from for other performance measures defined as gains. 

\section {A Novel Family of Listwise Surrogates}
\label{SLAM}
We define the novel $SLAM$ family of loss functions: these are Surrogate, Large margin, Listwise and Lipschitz losses, are Adaptable to Multiple ranking measures, and can handle Multiple graded relevance.

In RankSVM \cite{herbrich1999}, a loss is incurred on a pair of documents in a list, if a relevant document does not outscore an irrelevant document with a margin. We use this idea to develop the $SLAM$ family. In our definition of $SLAM$ loss function, we will use score vector $s \in \mathbb{R}^m$, corresponding to a list of $m$ documents, and relevance vector $R \in \mathbb{R}^m$. If the score vector is induced by linear scoring function, parameterized by $w$, as defined in Sec. \ref{probdef}, we write $s^w$ instead of $s$. The family of convex loss functions is defined as follows:
\begin{equation}
\label{eq:lossdef}
\begin{aligned}
\phi^v_{SLAM}(s,R) &= \min_{\delta \in \mathbb{R}^m} \sum_{i=1}^m v_i \delta_i \\
 \text{s.t.} & \ \ \ \  \delta_i \ge 0, \ \forall \ i \\
& \ \ \ \ s_i +\delta_i \ge \Delta + s_j, \ \text{if} \  R_i > R_j.
\end{aligned}
\end{equation}
Here, $\Delta$ is a margin-scaling constant and $v = (v_1,\ldots,v_m)$ is an element-wise non-negative weight vector yielding different members of the $SLAM$ family. Though $\Delta$ can be varied for empirical purposes, we fix $\Delta=1$ for subsequent analysis.

In a batch setting, the estimation of the parameter vector $w$ is  done via minimization of regularized empirical loss:
\begin{equation}
\label{eq:objfunc}
\begin{aligned}
\hat{w}\ = \underset{w}{\argmin}\left\{\frac{\lambda}{2 }\|w\|_{2}^2 + \frac{1}{n}\sum_{i=1}^{n}\phi^v_{SLAM}({(s^w)}^{(i)},R^{(i)})\right\}
\end{aligned}
\end{equation} 
where $\{(X^{(1)},R^{(1)}),\ldots,(X^{(n)},R^{(n)})\}$ are iid samples drawn from an unknown joint distribution on $\mathcal{X}\times \mathcal{Y}$. We point out  again that $(s^w)^{(i)}= f_{w}(X^{(i)})= X^{(i)}w \in \mathbb{R}^m$. \\

\begin{lem}
\label{convexity}
For any $R$, the function $\phi^v_{SLAM}(\cdot,R)$ is \emph{convex} 
\end{lem}

We prove this lemma in the appendix directly from the definition of convexity. However, the reformulation below makes it easy to see that $\phi^v_{SLAM}(\cdot,R)$ is convex:
{\small
\begin{equation}
\label{eq:theoreticalloss}
\begin{split}
& \phi^v_{SLAM}(s,R)= \\
& \sum_{i=1}^{m} v_i \  \max(0,\underset{j=1,\ldots,m}{ \max}\{I(R_i>R_j)(1+ s_j -s_i)\}) \ .
\end{split}
\end{equation}}

\subsection{Properties of the SLAM Family}
\label{properties}
The similarity with the RankSVM surrogate is understood by observing the constraints in Eq. \ref{eq:lossdef}. Similar to RankSVM, a loss is induced if a more relevant document fails to outscore a less relevant document with a margin. However, one of the the main modifications is that there is a \emph{single} $\delta_i$ corresponding to document $i$. Thus, unlike RankSVM, the loss is not added for each pair of documents. Rather, the maximum loss corresponding to each document and all documents less relevant than it is measured (as seen in Eq. \ref{eq:theoreticalloss}). Moreover, each $\delta_i$ is weighted by $v_i$ before they are added. The weight vector imparts a listwise nature to our surrogate.

As noted in Sec. \ref{introduction}, all popular ranking measures are \emph{listwise} in nature, where correct ranking at the top of the list is much more critical than near the bottom. The critical property that a surrogate must posses to be considered \emph{listwise} is this: the loss must be calculated viewing the entire list of documents as a whole, with errors at the top penalized much more than errors at the bottom. Since a perfect ranking places the most relevant documents at top, errors corresponding to most relevant documents should be penalized more in $SLAM$, in order for it to be considered a listwise surrogate family. The weight vector we design does exactly that. If document $i$ is the most relevant in the list, $v_i$ is the maximum entry in weight vector $v$. Thus, even though our loss definition uses intuitive pairwise comparison between documents, it is truly a listwise loss. We define two weight vectors, $v^{NDCG}$ and $v^{MAP}$, in Sec.\ref{upperbounds}. 

We want to re-emphasize the structural difference of $SLAM$ with listwise surrogates obtained via the structured prediction framework. There are multiple listwise surrogates in learning to rank literature. The popular large margin listwise surrogates are direct extensions of the structured prediction framework developed for classification \cite{tsochantaridis2004}. As we pointed out in Sec. \ref{introduction}, structured prediction for ranking models assume that the supervision space is the space of \emph{full rankings} of a document list. Usually a large number of full rankings are compatible with a relevance vector, in which case the relevance vector is arbitrarily mapped to a full ranking. In fact, here is a quote from one of the relevant papers \cite{chapelle2007}, ``It is often the case that
this $y_q$ is not unique and we simply take of one of them at random" ($y_q$ refers to a correct full ranking pertaining to query $q$). Though empirically they can yield competitive results; theoretically, structured prediction based ranking surrogates are less suitable in a learning to rank setting where supervision is given as relevance vectors but the ranking function returns full rankings.

\section{Weight Vectors Parameterizing the SLAM Family}
\label{upperbounds}
As we stated in Sec \ref{SLAM}, different weight vectors lead to different members of the $SLAM$ family. The weight vectors play a crucial role in the subsequent theoretical analysis.

First, the weight vectors need to be such that the surrogate family is truly listwise. For this, as explained in Sec \ref{properties}, maxiumum weights need to be assigned to most relevant documents.
Second, the weight vectors need to be such that different members of the $SLAM$ family are \emph{upper bounds} on (losses induced by) different ranking performance measures. The upper bound property will be crucial in deriving guarantees for a perceptron-like algorithm in learning to rank. Moreover, it makes sense to formally relate the loss being minimized to the performance measured being maximized. Recall that surrogates like hinge loss and logistic loss are upper bounds on the $0-1$ loss in classification. However, the weight vectors also need to be as small as possible, because the magnitude of the generalization bound for members of $SLAM$ ends up being directly proportional to sum of components of the corresponding weight vectors (see Sec. \ref{genbound}).

Thus, we will require weight vectors to be as small as possible so far as the corresponding members of $SLAM$ still upper bound different ranking performance measures. Upper bounds on ranking performance measure have also been investigated by \cite{chen2009}. However, our analysis technique is completely different, and yields different results.
 
We will provide two weight vectors, $v^{MAP}$ and $v^{NDCG}$, that results in upper bounds MAP and NDCG induced losses respectively. Since weight vectors are defined with the knowledge of relevance vectors, we can assume w.l.o.g that documents are sorted according to their relevance levels. Thus, $R_1 \ge R_2 \ge \ldots \ge R_m$, where $R_i$ is the relevance of document $i$. 

{\bf Upper bounding MAP loss}: It is to be noted that $MAP$ is defined for binary relevance vectors. Let $R \in \mathbb{R}^m$ be a binary relevance vector, where $r$ is the number of relevant documents (thus, $R_1=R_2=\ldots=R_r=1$ and $R_{r+1}=\ldots=R_m=0$). We define vector $v^{MAP} \in \mathbb{R}^m$ as 
\begin{equation}
\label{eq:mapweights}
v^{MAP}_i = \\
\left\{
	\begin{array}{ll}
		\frac{1}{r} - \frac{i}{r(m-r+i)}   & \mbox{if } i=1,2,\ldots,r\\
		0  & \mbox{if } i=r+1,\ldots,m .\\ 
           \end{array}
\right.
\end{equation}
We have the following theorem on upper bound.
\begin{thm}
\label{eq:upperbound1}
Let $v^{MAP} \in \mathbb{R}^m$  be the weight vector as defined in Eq. \ref{eq:mapweights}. Let $MAP(s,R)$  be the MAP value determined by  relevance vector $R \in \mathbb{R}^m$ and permutation induced by sorting of score vector $s \in \mathbb{R}^m$. Then the following holds,
\begin{equation}
\begin{aligned}
\forall R,\  \forall s, \ \phi^{v^{MAP}}_{SLAM}(s,R) \ge 1-MAP(s,R).
\end{aligned}
\end{equation}
\end{thm}

We say a vector $x \in \mathbb{R}^m$ dominates a vector $y \in \mathbb{R}^m$ ($x \prec y$) if $x_i\le y_i, \ \forall i$ and $x_j < y_j$ for at least one $j$.

For a given binary relevance vector $R \in \mathbb{R}^m$, let $F(R) = \{v \in \mathbb{R}^m: \forall s,\ \phi^v_{SLAM}(s,R) \ge 1- MAP(s,R)\}$. Then, $\forall \ R$, the following relation holds,

\begin{equation}
\label{admissibility}
\forall v \in F(R),\  v^{MAP} \prec v.
\end{equation}

We remind that $v^{MAP}$ is iteself a function of $R$.

Thus, the choice of $v^{MAP}$ makes it $optimal$ in the sense that it $dominates$ all other choices of upper-bounding weight vectors. This implies that $v^{MAP}$ leads to tightest possible upper bound on MAP induced loss when $\phi^v_{SLAM}$ is the surrogate used.
The proof of Eq. \ref{admissibility} follows as a direct consequence of the way $v^{MAP}$ is derived.

{\bf Upper bounding NDCG loss}: For a given relevance vector $R \in \mathbb{R}^m$, we define vector $v^{NDCG} \in \mathbb{R}^m$ as 
\begin{equation}
\label{eq:ndcgweights}
\begin{split}
v^{NDCG}_i & =  \\
& \frac{(G(R_i)- G(R_m))(D(i)- D(m))}{Z(R)}, \ i=1,\ldots,m .
\end{split}
\end{equation}
The definition of functions $G(\cdot), D(\cdot), Z(\cdot)$ are as given in Section \ref{probdef}.
We have the following inequality.
\begin{thm}
\label{eq:upperbound2}
Let $v^{NDCG} \in \mathbb{R}^m$  be the weight vector as defined in Eq. \ref{eq:ndcgweights}. Let $NDCG(s,R)$  be the NDCG value determined by  relevance vector $R \in \mathbb{R}^m$ and permutation induced by sorting of score vector $s \in \mathbb{R}^m$. Then the following inequality holds,
\begin{equation}
\begin{aligned}
\forall R,\  \forall s, \ \phi^{v^{NDCG}}_{SLAM}(s,R) \ge 1-NDCG(s,R) .
\end{aligned}
\end{equation}
\end{thm}

We note that the choice $v^{NDCG}$ is not optimal. However, the upper bound property still holds and it satisfies the condition required for $\phi^{v^{NDCG}}_{SLAM}(s,R)$ to have $m$-independent generalization bound (as detailed in Sec \ref{genbound}).

It can also be easily calculated that $\sum_{i=1}^m v^{NDCG}_i \le 1$ and $\sum_{i=1}^m v^{MAP}_i \le 1$. This fact will be crucial in the generalization bound analysis.


\section{Perceptron-like Algorithm for Learning to Rank}
\label{perceptron}

We present a perceptron-like algorithm for learning a ranking function in an online setting, using the $SLAM$ family. We also provide theoretical bounds on accumulated losses induced by two major ranking performance measures: NDCG and MAP. Though perceptron has been extended to a different ranking setting \cite{crammer2001}, to the best of our knowledge, cumulative loss guarantees for a perceptron-like algorithm (evaluated using popular performance measures such as NDCG and MAP) have not been provided before. The online gradient descent algorithm used in this section has been used by numerous authors (see the seminal paper of \cite{zinkevich2003online} and the survey article of \cite{shalev2011}).

Since our proposed perceptron like algorithm works for both NDCG and MAP induced losses, we denote a performance measure induced loss as \emph{RankingMeasureLoss} (RML). Thus, RML can be NDCG induced loss or MAP induced loss. 

To make subsequent calculations easy to understand, we re-write the $SLAM$ family from Eq.\ref{eq:theoreticalloss}. Also, we write $s^w$ for
$s$ to emphasize that we are using linear ranking functions.

Denoting $b_{ij}=\{I(R_i>R_j)(1 +s^w_j -s^w_i)\}$, we have 
\begin{equation}
\label{eq:surrogateinperceptron}
\begin{split}
&\phi_{SLAM}^v (s^w,R) =  \sum_{i=1} ^ {m} v_i\ c_i  \\
\end{split}
\end{equation}
where 
\begin{equation*}
c_i= \\
\left\{
	\begin{array}{ll}
		0   & \mbox{if } \underset{j=1,\ldots,m}{\max} b_{ij}\le 0\\
		1 + s^w_k -s^w_i \in \mathbb{R}  & \text{otherwise}\\ 
                      & k=  \underset{j=1,\ldots, m}{\argmax}\ {b_{ij}}.\\
	\end{array}
\right.
\end{equation*}
 
It is easy to see Eq.\ref{eq:surrogateinperceptron} and Eq.\ref{eq:theoreticalloss} are the same. 
We remind the reader that for our choice of weight vectors $v^{NDCG}$ and $v^{MAP}$ as defined in Eq.\ref{eq:ndcgweights} and Eq.\ref{eq:mapweights} respectively, we have, $\forall \ s^w, \ \forall \ R$, the following inequalities,
\begin{equation}
\label{eq:upperboundinperceptron}
\begin{aligned}
&\phi_{SLAM}^{v^{NDCG}} (s^w,R) \ge 1- NDCG(s^w,R)\\
&\phi_{SLAM}^{v^{MAP}} (s^w,R) \ge 1- MAP(s^w,R)
\end{aligned}
\end{equation}
It should also be noted that $v^{NDCG}$ and $v^{MAP}$ are functions of $R$.

In the online learning setting, at round $t$, the input received is $X_t$ and ground truth received is $R_t$. We define the following function
\begin{equation}
\label{eq:functioninperceptron}
f_t(w)= \\
\left\{
	\begin{array}{ll}
		\phi_{SLAM}^{v_t} (s^{w}_t,R_t)  & \mbox{if } RML(s^{w_t}_t,R_t) \neq 0\\
		0 & \mbox{if } RML(s^{w_t}_t,R_t) = 0\\ 
         
	\end{array}
\right.
\end{equation}
Here, $w_t$ is the function parameter learnt at time point $t$, $s^{w}=X_tw$ and $v_t=v_t^{NDCG}$ or $v_t^{MAP}$ depending on whether $RML$ is NDCG induced loss or MAP induced loss respectively.  Since weight vector $v$ depends on relevance vector $R$, $v_t$ depends on $R_t$. 

It is clear from Eq.\ref{eq:upperboundinperceptron} and Eq.\ref{eq:functioninperceptron} that $f_t(w_t) \ge RML(s^{w_t}_t,R_t)$.
It should also be noted that  that $f_t(\cdot)$ is convex in both cases, i.e, when $RML(s^{w_t}_t,R_t) \neq 0$ and $RML(s^{w_t}_t,R_t) = 0$. 
Due to the convexity of the sequence of functions $f_t$, we can run online gradient descent (OGD) algorithm to learn the sequence of parameters $w_t$, starting with $w_1=\mathbf{0}$. The OGD update rule, $w_{t+1}= w_t - \eta z_t$, for some $z_t \in \partial{f_t}(w_t)$ and step size $\eta$, requires a sub gradient $z_t$ that, in our case, is:

When $RML(s^{w_t}_t,R_t)=0 \implies z_t=0 \in \mathbb{R}^d$.

When $RML(s^{w_t}_t,R_t) \neq 0 \implies$
\begin{equation}
\label{eq:gradientinperceptron}
\begin{split}
& z_t=  X^{\top}_t(\sum_{i=1} ^ {m} v^t_i\ a^t_i) \in \mathbb{R}^d  \\
\end{split}
\end{equation}
where 
\begin{equation*}
a^t_i= \\
\left\{
	\begin{array}{ll}
		\mathbf{0} \in \mathbb{R}^m   & \mbox{if } c^t_i=0\\
		\mathbf{e}_k - \mathbf{e}_i \in \mathbb{R}^m  & \mbox{if }  c^t_i \neq 0\\ 
                     
	\end{array}
\right.
\end{equation*}
Here, $\mathbf{e}_k$ is the standard basis vector along coordinate $k$ and $c^t_i$ is as defined Eq.\ref{eq:surrogateinperceptron} (with $w = w_t$).

Note that $RML(s^{w_t}_t,R_t) \neq 0$ means that there is \emph{at least one} document with relevance less than \emph{at least another} document but with greater score. That is, there is at least one pair of documents, indexed by $(i,j)$, with $R_{t,i} > R_{t,j}$ but $s^{w_t}_{t,j} > s^{w_t}_{t,i}$.


Since predicted ranking at round $t$ is obtained by sorting the score vector $s^w_t$, we have, from the update rule,
the following prediction at round $t$
$$Pred_t = sort(X_tw_t) =  sort \left(-\eta \left(\underset{i<t, i \in M}{\sum} X_i z_i \right) \right)$$
where $M$ is the set of rounds on which $RML(s^{w_t}_t,R_t )\neq 0$.
Since sorted order of a vector is invariant under scaling by a positive constant, $Pred_t$ and $M$ do not depend on $\eta$ as long as $\eta >0$. Thus, we can take $\eta=1$ in our algorithm.
We now obtain a perceptron-like algorithm for the learning to rank problem.
\floatstyle{ruled}
\newfloat{algorithm}{htbp}{loa}
\floatname{algorithm}{Algorithm}
\begin{algorithm}
\caption{Perceptron Algorithm for Learning to Rank}
\label{alg:LA}
\begin{tabbing}
{\bf Initialize} $w_1=\mathbf{0} \in \mathbb{R}^d$\\
{\bf For} \=$t=1$ to $T$ \\
\> Receive $X_t$\\
\> Set $s^{w_t}_t = X_tw_t$ \& predict $Pred_t= sort(s^{w_t}_t)$\\
\> Receive $R_t$ \\
\> {\bf If} $RML(s^{w_t}_t, R_t) \neq 0$ \= \ \footnotemark\\
\> \ \ $w_{t+1} =w_t - z_t$ \> // see def. of $z_t$ in Eq.\eqref{eq:gradientinperceptron}\\
\> {\bf else}\\
\> \ \ $w_{t+1} =w_t $\\
{\bf End For}
\end{tabbing}
\end{algorithm}
\footnotetext{The first argument in RML is actually the sorted order of $s^{w_t}_t$, as detailed in Sec.\ref{probdef}. Thus, the $if-else$ condition of the algorithm depends on $Pred_t$}

\subsection{Theoretical Bound on Cumulative Loss}

We provide a theoretical  bound on the cumulative loss (as measured by RML) of perceptron for the learning to rank problem. This result is similar to the theoretical bound on accumulated $0$-$1$ loss of classic perceptron in the binary classification problem. The technique is based on regret analysis of online convex optimization algorithms.
In this analysis, $\|\cdot\|$ is used to represent the Euclidean norm (or $l_2$ norm), unless otherwise stated.
We begin by stating a standard bound from the literature \cite{zinkevich2003online,shalev2011}.
\begin{propositionOCO}
\label{regretboundinperceptron}
Let $f_t$ be parameterized by any $u \in \mathbb{R}^d$. Then the following regret bound holds for OGD, after $T$ rounds,
\begin{equation}
\sum_{t=1}^T f_t(w_ t)\ -\sum_{t=1}^T f_t(u) \ \le \ \frac{\|u\|^2}{2\eta} + \frac{\eta}{2} \sum_{t=1}^T \|z_t\|^2
\end{equation}
where $\eta$ is the learning parameter and $z_t \in \partial{f_t}(w_t)$.
\end{propositionOCO}

We first control the norm of the subgradient $z_t$.

\begin{proposition}
\label{gradientboundinperceptron}
Let $R_X$ be the bound on the maximum $l_2$ norm of the feature vectors, as defined in Sec. \ref{probdef}. Let $v^t_{max}= \underset{i,j}{\max}\frac{v^t_i}{v^t_j}, \ \forall \ i,j$ with $v^t_i>0,\  v^t_j>0$. Then the following $l_2$ norm bound on the subgradient holds,
\begin{equation}
\|z_t\|^2 \le 4 m R_X^2 v^t_{max} f_t(w_t), \forall \ t
\end{equation}
\end{proposition}

Assuming that  $\max_{t=1}^T v^t_{max} \le v_{max}$ setting taking $\eta = \tfrac{1}{4 m R_X^2 v_{max}}$, we have our main theorem for the proposed perceptron algorithm. Note that since $RML(s^{w_t}_t,R_t)$ is independent of $\eta > 0$, the same bound holds for Algorithm~\ref{alg:LA} even though it uses $\eta = 1$.
\begin{thm}
\label{theoryboundinperceptron}
Suppose the perceptron algorithm receives a sequence of instances ${(X_1,R_1),\ldots,(X_T,R_T)}$. Let $R_X$ be the bound on the maximum $l_2$ norm of feature vectors. Then for $RML$ defined in Sec.\ref{perceptron}, $f_t$ defined in Eq.\ref{eq:functioninperceptron}, $v_{max} \ge \max_{t=1}^T v^t_{max}$, and $m$ being the bound on number of documents per query, the following bound holds.
\begin{equation}
\label{eq:perceptronmistakebound}
\sum_{t=1}^T RML(s^{w_t}_t,R_t)\ \le 2 \ \sum_{t=1}^T f_t(u) + \ 4 \|u\|^2  m R_X^2 v_{max}
\end{equation}
In particular, if there exists an $u$ s.t. $f_t(u)=0 \ \forall \ t$, we have,
\begin{equation}
\label{eq:perceptronmistakebound1}
\sum_{t=1}^T RML(s^{w_t}_t,R_t)\ \le \ 4\|u\|^2 m R_X^2 v_{max}, \forall \ T.
\end{equation}
\end{thm}

The perceptron RML bound in Eq.\ref{eq:perceptronmistakebound} is meaningful only if $v_{max}$ is a meaningful, finite quantity. 
It can be seen from the definition of $v^{MAP}$ in Eq.\ref{eq:mapweights} that $v_{max} \le \tfrac{m}{2}$. Thus, when $RML$ is MAP induced loss, the perceptron bound is meaningful and is $O(m^2)$ (hiding the $\|u\|^2R_X^2$ dependence).
For $v^{NDCG}$, $v_{max}$ depends on maximum relevance level. Assuming maximum relevance level is finite (in practice, maximum relevance level is usually between $2$ and $5$), $v_{max}= O(m\ (\log(m))^2)$. Thus, when $RML$ is NDCG induced loss, the perceptron bound is meaningful and is $O(m^2 (\log(m))^2)$.

Like perceptron for binary classification, the bound is Eq. \ref{eq:perceptronmistakebound1} leads to an interesting conclusion. Let us assume that there is a linear scoring function parameterized by a unit vector $u_\star$, such all documents for all queries are ranked not only correctly, but correctly with a \emph{margin} $\gamma$:
\[
\min_{t=1}^T \min_{i,j: R_{t,i} > R_{t,j}} u_\star^\top X_{t,i} - u_\star^\top X_{t,j} \geq \gamma .
\]

\begin{cor}
\label{cor:margin}
If the margin condition above holds, then accumulated losses, for both NDCG and MAP induced loss, is upper bounded by $4 m R_X^2 v_{max}/\gamma^2$, a constant independent of the number of training instances.
\end{cor}

We point out that the bound on the cumulative loss in Eq. \ref{eq:perceptronmistakebound1} is dependent on $m$. It is often the case that though a list has $m$ documents, the focus is on the top $k$ documents in the order sorted by score. We define a modified set of weights $v^{NDCG@k}$ s.t. $\phi_{SLAM}^{v^{NDCG@k}}(s,R) \ge 1- NDCG(s,R)@k$ holds $\forall \ R, \ \forall \ s$. We provide the definition of $NDCG(s,R)@k$ and $v^{NDCG@k}$ in the appendix. We note that $\sum_{i=1}^m v^{NDCG@k}_i =1$.

Overloading notation with $v^t= v^{NDCG@k, t}$, let $v^{t}_{max}= \underset{i,j}{\max}\dfrac{v^t_i}{v^t_j}$ with $v^t_i>0,\  v^t_j>0$ and $v_{max} \ge \max_{t=1}^T v^t_{max}$. 

\begin{cor} 
\label{cor:k-dependence}
In the setting of Theorem \ref{theoryboundinperceptron} and $k$ being the cut-off point for NDCG, the following bound holds
\begin{equation}
\label{eq:perceptronmistakeboundfork}
\sum_{t=1}^T (1-NDCG(s^{w_t}_t,R_t)@k)\ \le 2 \ \sum_{t=1}^T f_t(u) + \ 4 \|u\|^2  k R_X^2 v_{max}
\end{equation}
\end{cor}
Assuming maximum relevance level is finite, $v_{max}= O(log(k))$. Thus, the variance term in the perceptron bound is $O(k)$, a significant improvement from original variance term.

\section{Generalization Error Bound}
\label{sufficientcondition}
In batch setting, the ranking function parameter $w$ is learnt by solving Eq.\ref{eq:objfunc}. We analyze how ``good" the learnt parameter is w.r.t. to the functional parameter minimizing expected $\phi^v_{SLAM}$. We formalize this notion via establishing a generalization error bound. 

Our main theorem on generalization error is applicable to \emph{any} convex ranking surrogate with linear ranking function. But first, we take a closer look at the concept of a ``linear ranking function" that is prevalent in the learning to rank literature, and show that it is actually a low dimensional parameterization of the full space of linear ranking functions.

As stated in Sec.\ref{probdef}, ranking is obtained by sorting a score vector obtained via a linear scoring function $f_w$. Specifically, $w\in \mathbb{R}^d$ is a $d$ dimensional vector which maps the matrix $X \in \mathbb{R}^{m\times d}$ to a $m$ dimensional score vector $s \in \mathbb{R}^m$. The space of linear scoring function consists of linear maps $f: \mathbb{R}^{m \times d} \rightarrow \mathbb{R}^m$. The linear function space can be fully parameterized by matrices $(W_1,\ldots,W_m)$, where $W_i \in \mathbb{R}^{m \times d}$. The representation will be of the form
$$
f(X)= [\langle{X,W_1}\rangle, \ldots, \langle{X,W_m}\rangle]^{\top} \in \mathbb{R}^m ,
$$
where $\langle{X,W}\rangle := \mathrm{Tr}(W^\top X)$. Thus, a full parameterization of the linear scoring function is of dimension $m^2 \times d$.

The popularly used form of linear scoring function, viz. $f(X)= Xw \in \mathbb{R}^m$, with $w \in \mathbb{R}^d$ is actually a low $d$-dimensional subspace of the full $m^2d$ dimensional space of linear maps. It corresponding to choosing each matrix $W_i$ such that the $i$th row is the vector $w \in \mathbb{R}^d$ and rest of the rows are vectors $\bf{0} \in \mathbb{R}^d$. Thus, it is a $d$-dimensional parameterization. \emph{Most importantly, the dimension is independent of $m$}.

In learning theory, one of the factors influencing the generalization error bound is the richness of the class of hypothesis functions. Since the parameterization of the linear ranking function is of dimension independent of $m$, intuition would suggest that, under some conditions, ranking surrogates with linear ranking function should have an $m$ independent complexity term in the generalization bound. 

Before we state our main theorem on generalization error bound, we need some notations. For input matrix $X \in \mathcal{X}$, relevance vector $R \in \mathcal{Y}$, weight vector $w \in \mathbb{R}^d$ and any convex (in first argument) surrogate loss function $\ell(s^w,R)$, we denote
\begin{equation}
L(w) = \E{\ell(s^w,R) }
\end{equation}
where $s^w=Xw \in \mathbb{R}^m$. The expectation is taken over the underlying joint distribution on $\mathcal{X}\times \mathcal{Y}$. We also define
\begin{equation}
\label{eq:minimizer}
w^\star= \argmin_w \ L(w) 
\end{equation}
and
\begin{equation}
\label{eq:empminimizer}
\hat{w}\ = \underset{w}{\argmin}\left\{\frac{\lambda}{2 }\|w\|_{2}^2 + \frac{1}{n}\sum_{i=1}^{n}\ell({(s^w)}^{(i)},R^{(i)})\right\}
\end{equation}
where $((X^{(1)},R^{(1)}),\ldots,(X^{(n)},R^{(n)}))$ are iid samples from the underlying joint distribution on $\mathcal{X}\times \mathcal{Y}$.

We now have our main theorem on generalization error bound.
\begin{thm}
\label{genbound}
Let $w \mapsto \ell(s^w,R)$ be convex and Lipschitz continuous w.r.t. $w$ in the $l_2$ norm with constant $L_2$. Let $\hat{w}$  and $w^\star$ be defined as in Eq. \ref{eq:empminimizer} and Eq. \ref{eq:minimizer} respectively with the further restriction that $\|w\|_2 \le B$. Then, with a sample size of $n$, and with $\lambda = O(1/\sqrt{n})$, we have
\begin{equation}
\begin{split}
\E{L(\hat{w}) } \le  L(w^\star)  + 2\, L_2\, B\left(\frac{8}{n} + \sqrt{\frac{2}{n}}\right)
\end{split}
\end{equation}
where the expectation is taken over input sample $((X^{(1)},R^{(1)}),\ldots,(X^{(n)},R^{(n)}))$.
\end{thm}

Lipschitz continuity of $\ell(s^w,R)$ w.r.t $w$ in $l_2$ norm means that there is a constant $L_2$ such that $|\ell(s^w,R)- \ell(s^{w'},R)| \le L_2 \|w-w'\|_2$, for all $w, w'\in\reals^m$. By duality, it follows that $L_2 \ge \underset{w}{\sup}\ \|\nabla_{w}\ell(s^w,R)\|_2$. 
Now, by chain rule, we have $\|\nabla_{w}\ell(s^w,R)\|_2= \|X^{\top} \nabla_{s^w}\ell(s^w,R)\|_2$. It turns out that if each row of $X$ is bounded in
$l_2$ norm by $R_X$ and $\| \nabla_{s^w}\ell(s^w,R) \|_1 \le \widetilde{L}_1$ then $\|X^{\top} \nabla_{s^w}\ell(s^w,R)\|_2 \le R_X \widetilde{L}_1$
and the bound in Theorem~\ref{genbound} becomes $O(\widetilde{L}_1 B R_X/\sqrt{n})$.
This immediately gives the following corollary, which provides a sufficient condition for an $m$ independent generalization bound to hold.

\begin{cor}
\label{sufficient}
A sufficient condition for the ranking surrogate $\ell(s^w,R)$ to have $m$ independent generalization bound is for it have $m$ independent Lipschitz bound, w.r.t $s^w$, in $l_{\infty}$ norm. That is, there is a constant $\widetilde{L}_1$, independent of $m$, such that $\widetilde{L}_1 \ge \underset{s^w}{\sup}\ \|\nabla_{s^w}\ell(s^w,R)\|_1$
\end{cor}

We point out that the generalization bound in Theorem \ref{genbound} depends on the Lipschitz constant of $\ell(\cdot,R)$ w.r.t $w$. However, the condition in Corollary \ref{sufficient} depends on the Lipschitz constant of $\ell(\cdot,R)$ w.r.t $s^w$ (the tilde in $\widetilde{L}$ serves a reminder
that Lipschitz continuity is meant w.r.t. $s^w$, not $w$).

The only comparable result in the existing literature is the generalization bound given by \cite{chapelle2010} for ranking surrogates with linear ranking function. Their generalization bound is $O(\widetilde{L}_2 B R_X \sqrt{m/n})$, where $\widetilde{L}_2$ is the Lipschitz constant of the surrogate w.r.t $s^w$ in $l_2$-norm. The generalization bound, however, is inherently dependent on $m$ and ours is \emph{always} better since $\widetilde{L}_1 \le \sqrt{m} \widetilde{L}_2$. A comparison of the proof techniques reveals that \cite{chapelle2010} proceed via Gaussian complexity and use Slepian's lemma that forces them to use $l_2$ Lipschitz constant and introduces the $\sqrt{m}$ dependence. We use stochastic convex optimization results of \cite{shalev2009} thereby avoiding the explicit $\sqrt{m}$ dependence. However, the price we pay is that our result only holds for \emph{convex} surrogates whereas that of \cite{chapelle2010} holds for any Lipschitz surrogate.

We also note that \cite{lan2009} obtained generalization bounds for certain listwise surrogates. However, their analysis technique went via Rademacher complexity theory and is limited to specific listwise surrogates, while ours is a general result, applicable to all convex surrogates using linear ranking function.

We now show that $SLAM$ family satisfies the sufficient condition.
Let $ b_{ij}= \{I(R_i>R_j)(1 +s^w_j -s^w_i)\}$.  The gradient of $\phi^v_{SLAM}(s^w,R)$ w.r.t. to $s^w$, is as follows:\\
\begin{equation}
\label{eq:grad}
\begin{split}
&\nabla_{s^w}{\phi^v_{SLAM}(s^w,R)} =  \sum_{i=1} ^ {m} v_i\ a^i  \\
\end{split}
\end{equation}
where 
\begin{equation*}
a^i = \\
\left\{
	\begin{array}{ll}
		\mathbf{0} \in \mathbb{R}^m   & \mbox{if } \underset{j=1,\ldots,m}{\max}b_{ij} \le 0\\
		\mathbf{e}_k - \mathbf{e}_i \in \mathbb{R}^m  & \text{otherwise}\\ 
                      & k=  \underset{j=1,\ldots, m}{\argmax}\ {b_{ij}}\\
	\end{array}
\right.
\end{equation*}
and $\mathbf{e}_i$ is a standard basis vector along coordinate $i$.

Since $\|a^i\|_1 \le 2$, we have $\|\nabla_{s^w}{\phi^v_{SLAM}(s^w,R)}\|_{1} \le 2\sum_{i=1}^m v_i$. Further, $\sum_{i=1}^m v^{NDCG}_i$ and $\sum_{i=1}^m v^{MAP}_i $ are both bounded by $1$. Hence the $SLAM$ family members corresponding to both NDCG and MAP have $m$-independent generalization bound.

We now go back and analyze why the linear scoring function $f(X)= Xw \in \mathbb{R}^m$, with $w \in \mathbb{R}^d$ is the \emph{only} correct choice in the learning to rank setting. Though we mentioned that the full parameterization of the linear scoring function is of dimension $m^2 \times d$, we will formally prove that the correct full parameterization, under a natural permutation invariance condition, is of dimension $d$. 

An important property in ranking is permutation invariance. This means that score assigned to documents should be independent of the order in which documents are listed. Formally, a linear scoring function can be used for ranking if it satisfies the permutation invariance property:
$$ \forall \ \pi \in S_m,\ \forall \ X \in \mathbb{R}^{m \times d},\ \pi f(X) = f(\pi X) .
$$
We now show that the vector space of linear function that satisfy the permutation invariance property has dimension no more than $d$. Because functions of the form $f(X) = Xw$ are obviously permutation invariant and constitute a space of dimension $d$, we easily then get that the dimension has to be exactly $d$.

\begin{thm}\label{thm:perminvdim}
The space of linear, permutation invariant functions from $\mathbb{R}^{m \times d}$ to $\mathbb{R}$ has dimension at most $d$.
\end{thm}
{\bf Proof}:  Using the full parameterization model, the permutation invariance property translates into: $P[\langle{X,W_1}\rangle, \ldots, \langle{X,W_m}\rangle]= [\langle{PX,W_1}\rangle, \ldots, \langle{PX,W_m}\rangle], \forall \ P$, where $P$ is permutation matrix of order $m$.

Let $\rho_1= \{P: \pi_P(1)= 1 \}$, where $\pi_P(i)$ denotes the index of the element in the $i$th position of the permutation induced by $P$. Then, $\forall \ P \in \rho_1$, $\langle{X,W_1}\rangle= \langle{PX,W_1}\rangle$. Using $\forall C, \mathrm{Tr}(AC)= \mathrm{Tr}(BC) \implies A=B$, we get $W_1^{\top}= W_1^{\top}P$. Since $P$ will preserve the first column and create any permutation of the other columns, this indicates that all columns of $W_1^{\top}$ are same, except maybe the first column. We can repeat this arguement for $W_i, i=\{2,\ldots,m\}$.

Let $\rho_2= \{P: \pi_P(1)= 2 \}$. Then, $\forall \ P \in \rho_2$, $\langle{X,W_2}\rangle= \langle{PX,W_1}\rangle \implies W_2^{\top} = W_1^{\top}P$. $P$ will put the second column of $W_1$ in first position and create any other permutation of the other columns. Hence, the first column of $W_2^{\top}$ will match the second column of $W_1^{\top}$,  and the second column of $W_2^{\top}$ will match both first and thrid column of $W_1^{\top}$. Hence, all columns of matrix $W_1^{\top}$ and $W_2^{\top}$ are same and the matrices themselves are same. The argument can be repeated to show $W_1^{\top}=W_2^{\top}=\ldots=W_m^{\top}$ and $W_i^{\top}$ is a rank 1 matrix. 

Hence the linear function space has maximum dimension of $d$.

\subsection{Application to Existing Surrogates}
\label{applications}

In this subsection, we  check whether a few popular convex ranking surrogates, which learn linear ranking function, satisfy the sufficient condition established above. We select only a few from the plethora of surrogates existing in learning to rank literature,  representing both pairwise and listwise surrogates. All relevant calculations are shown in the appendix.

RankSVM minimizes a pairwise large margin surrogate and is designed for binary relevance vector. The $l_1$ norm of the gradient, w.r.t. score vector, is $O(m)$ and hence fails to satisfy the sufficient condition for $m$ independent generalization bound.

ListNet optimizes a listwise cross-entropy loss (as the surrogate) in conjunction with linear ranking function. Our calculations show that the surrogate is Lipschitz in $l_1$ norm, w.r.t. score vector, and is independent of $m$. The $l_1$ Lipschitz is actually bounded by the constant $2$. However, we point out that since the surrogate is not large-margin in nature, its use in online learning will not result in a perceptron-like algorithm. The gradient of the surrogate varies with the point where the gradient is calculated, which makes the online predictions sensitive to the choice of the learning rate $\eta$.

We also analyze large margin listwise surrogates suggested by \cite{chapelle2007} and \cite{yue2007}, which are realizations of structured prediction framework. To make the surrogates theoretically suitable for learning to rank problem, we assume relevance levels within each relevance vector to be distinct. This gives a one-one mapping from space of relevance scores to space of full ranking, without any arbitrariness.

\cite{chapelle2007} minimize a listwise large margin surrogate and can handle multi-graded relevance vector. The $l_1$ norm of the gradient, w.r.t. score vector, is $O(m^2)$ and hence fails to satisfy the sufficient condition for $m$ independent generalization bound. If the $m$ dependence is removed by simple normalization, the loss does not remain an upper bound on NDCG induced loss.

\cite{yue2007} minimize a listwise large margin surrogate and is designed for binary relevance vector. The $l_1$ norm of the gradient, w.r.t. score vector, is constant and hence satisfy the sufficient condition for $m$ independent generalization bound.

\section{Conclusion}

We provided the first perceptron-like algorithm for learning to rank that enjoys guaranteed loss bounds under losses induced by ranking performance measures such as NDCG and MAP. The loss bounds become independent of the number of training examples under a suitable margin condition. We also provided generalization bounds for general convex surrogate loss functions with linear ranking functions. Our analysis implied a sufficient condition for having a generalization bound that does not scale with $m$, the number of documents per query. A key role in both the online bounds and generalization bounds is played by a novel family of listwise surrogates that we introduced in this paper by modifying a well known pairwise surrogate.

Several interesting questions for further exploration are suggested by our results. First, is it possible to derive a perceptron-like algorithm whose cumulative loss bound (under NDCG or MAP induced losses) does not scale with $m$? Second, is it possible to extend our main generalization bound to all Lipschitz surrogates and not just convex ones? Third, do the online and batch algorithms implied by our novel loss family enjoy good practical performance possibly with the use of kernels to tackle non-linearities? Our preliminary experiments suggest that it is the case but a full empirical comparison with the existing state-of-the-art is outside the scope of this paper and will be pursued in a subsequent work.

\section*{Acknowledgments}

We gratefully acknowledge the support of NSF under grant IIS-1319810. Thanks to Prateek Jain for pointing out the simple argument required to prove
Theorem~\ref{thm:perminvdim}.

\bibliography{RankingTheory}
\bibliographystyle{plain}
\newpage

In the appendix, we provide proofs of theorems stated in the main section. Unless otherwise stated, $s$ and $s^w$ are used alternatingly, with $w$ understood from the context. 

\subsection{Proof of Lemma \ref{convexity}}
\begin{proof} 
Let $C(R)=\{(s,\delta),\ s \in \mathbb{R}^m,\ \delta \in \mathbb{R}^m\ |(s,\delta)$ satisfies the constraints of Eq. \eqref{eq:lossdef}\}. Then $C(R)$ defines a polyhedra and hence is a convex set.\\
Let $g(\delta) = v^{\top}\delta= \sum_{i=1}^m v_i \delta_i$, for some non-negative vector $v$. Thus, $g(\delta)$ is a convex function.
Let us formulate a function h as follows:
\begin{equation}
h(s,\delta) =
\left\{
	\begin{array}{ll}
		g(\delta)  & \mbox{if } (s,\delta) \in C(R) \\
		\infty & \text{otherwise}
	\end{array}
\right.
\end{equation}
We will first show that $h(s,\delta)$ is a jointly convex function.

For joint convexity, we have to show that:
 $h(\lambda (s_1,\delta_1) + (1-\lambda)(s_2,\delta_2))$ $\le$ $\lambda h(s_1,\delta_1)$ + $(1-\lambda)h(s_2,\delta_2)$, for $0 \le \lambda \le 1$.

Let $(s_1,\delta_1) \in C(R)$, $(s_2,\delta_2) \in C(R)$.  (If either of the vectors is not in $C(R)$, then the right side of convexity equation is $\infty$ and the inequality is trivially true). Then $\{\lambda(s_1,\delta_1) +(1-\lambda)(s_2,\delta_2)\} \in C(R)$, since $C(R)$ is a convex set. Hence,   $h(\lambda (s_1,\delta_1) + (1-\lambda)(s_2,\delta_2))$= $g(\lambda \delta_1 +(1-\lambda)\delta_2)$ $\le$ $\lambda g(\delta_1) +(1-\lambda) g(\delta_2)$ = $\lambda h(s_1,\delta_1) +(1-\lambda)h(s_2,\delta_2)$ (due to convexity of g).

As $h(s,\delta)$ is jointly convex, and $\phi_{SLAM}$ is the minimum of $h(s,\delta)$ over $\delta$ in a convex set C(R), $\phi_{SLAM}$ is convex in $s$ \cite{boyd2004}. 
\end{proof}

\subsection{Proof of Theorem \ref{eq:upperbound1}}

\begin{proof}
As stated in Sec.\ref{upperbounds}, documents pertaining to every query is sorted according to relevance labels. Let $R \in \mathbb{R}^m$ be an arbitrary relevance vector, corresponding to $r$ relevant documents and $m-r$ irrelevant documents in a list. MAP loss is only incurred if atleast 1 irrelevant document is placed above atleast 1 relevant document. With reference to $\phi^v_{SLAM}$ in Eq. \ref{eq:theoreticalloss}, for any $i \ge r+1$ and $\forall \ j>i$,  we have $I(R_i>R_j)=0$, since documents are sorted according to relevance labels. Thus, w.l.o.g, we can take $v_{r+1},..., v_m=0$.

Let a score vector $s$ be such that an irrelevant document $j$  has the highest score among $m$ documents. Then, $\phi^v_{SLAM}= v_1(1 +s_j-s_1) +v_2(1+s_j-s_2) +...+v_r(1+s_j-s_r)$.  The maximum possible MAP induced loss in case atleast one irrelevant document has highest score is when all irrelevant documents outscore all relevant documents. The MAP loss in that case is: $ 1- \frac{1}{r}(\frac{1}{m-r+1} + \frac{2}{m-r+2}+..+ \frac{r}{m-r+r})$. Since $\phi^v_{SLAM}$ has to upper bound MAP $\forall \ s \ and\ \forall \ R$ and since $s_j$ can be infinitesimally greater than all of $\{s_1,...,s_r\}$ (thus, $1+s_j-s_i \sim 1, \ \forall \ i=1,\ldots,r$), we need the following equation for upper bound property to hold:

$v_1 + v_2+...+ v_r \ge  1- \frac{1}{r}(\frac{1}{m-r+1} + \frac{2}{m-r+2}+..+ \frac{r}{m-r+r})$.

Similarly, let a score vector $s$ be such that an irrelevant document $j$ has higher score than all but the 1st relevant document. Then $\phi^v_{SLAM}= v_2(1 +s_j-s_2) +v_3(1+s_j-s_3) +...+v_r(1+s_j-s_r)$. The maximum possible MAP induced loss in case atleast one irrelevant document has higher score than all but 1st relevant document is when all irrelevant documents are placed above all relevant documents but first one. The MAP loss in that case is: $ 1- \frac{1}{r}(1+ \frac{2}{m-r+2} + \frac{3}{m-r+3}+..+ \frac{r}{m-r+r})$. Following same line of logic for upper bounding as before, we get

$v_2 + v_3+...+ v_r \ge  1- \frac{1}{r}(1+ \frac{2}{m-r+2} + \frac{3}{m-r+3}+..+ \frac{r}{m-r+r})$.

Likewise, if we keep repeating the logic, we get sequence of inequalities, with the last inequality being

$v_r \ge 1 - \frac{1}{r}(r-1 + \frac{r}{m-r+r})$.

To get smallest possible $v_i$'s, we take equality in all equations and by back calculation, we get $v= v^{MAP}$.

{\bf Proof of dominance}: Let, for some $R$, $v \in F(R)$ s.t. $v^{MAP} \nprec v$. Thus, $\exists \ k,m$ s.t. $v^{MAP}_k < v_k$ but $v^{MAP}_m > v_m$. However, if we assume $v_k = v^{MAP}_k + \epsilon$ and $v_m = v^{MAP}_m - \epsilon_1$, then the inequality $ v_k(1 +s_j-s_k) +\cdots + v_{m}(1+s_j-s_m) +\cdots+v_r(1+s_j-s_r) \ge R.H.S$, can fail. This happens when $\epsilon*v_k(1+s_j -s_k) > \epsilon_1* v_m(1+s_j-s_m)$ and the fact that there is no way to control $s_j-s_k$ and $s_j-s_m$.
\end{proof}

\subsection{Proof of Theorem \ref{eq:upperbound2}}

\begin{proof}
Using NDCG definition given in Eq. \ref{eq:NDCG}, we get
\begin{equation*}
\begin{split}
& \quad 1- NDCG(s,R) \\
& = \frac{1}{Z(R)} \sum_{i=1}^{m} G(R_i) D(i) - \frac{1}{Z(R)}\sum_{i=1}^m G(R_i)D(\pi^{-1}_s(i)) \\
& = \frac{1}{Z(R)} \sum_{i=1}^{m} G(R_i) \left( D(i) - D(\pi_s^{-1}(i)) \right) 
\end{split}
\end{equation*}
We define $NDCGL(s,R)= 1- NDCG(s,R)$.  We make 2 important observations:

For any $s \ and \ R$, document $i$ increases value of $NDCGL(s,R)$ iff $\pi^{-1}_s(i) > i$, i.e, document $i$ is placed below position $i$ in permutation $\pi_s$. This can be seen from the definition of $NDCGL(s,R)$ and noting that $D(\cdot)$ is a decreasing function. We also note that $G(\cdot)$ is an increasing function.

For any $s \ and \ R$, each summation term in $\phi^v_{SLAM}(s,R)$ is non-negative.

Let $\{i_i,i_2,\ldots,i_k\}$ be the indices where outer \emph{max} function in $\phi^v_{SLAM}(s,R)$ is greater than 0 (which implies that the inner \emph{max} function$>$ 0). W.l.o.g, we take $i_1>\ldots>i_k$.

Thus, for every $i \notin \{i_i, \ldots ,i_k\}$, we have $\underset{j=i+1,\ldots,m}{ \max}\{I(R_i>R_j)(1.(R_i-R_j) + s_j -s_i) \le 0$. Thus, for each $j=i+1,\ldots,m$, either $R_i=R_j$ or $s_j < s_i$.

A necessary condition for document $i$ to be placed below position $i$ by $\pi_s$ is that there is a document $j$, with $j>i$, s.t $s_j > s_i$. Thus, for $i \notin \{i_i,\ldots,i_k\}$, document $i$ cannot be placed below position $i$ in permutation $\pi_s$, and thus cannot increase $NDCGL$ (If for some $i \notin \{i_i,\ldots,i_k\}$, it happens that $s_j>s_i$ for some $j>i$, it means $R_i=R_j$ and we can consider document $i$ and $j$ exchanged in the original sorted list).

For $i_L \in \{i_1,\ldots,i_k\}$, we have $\underset{j=i_L+1,...,m}{ \max}\{I(R_{i_L}>R_j)(1.(R_{i_L}-R_j) + s_j -s_{i_L}) > 0 \implies (R_{i_L} - R_j) + s_j > s_{i_L}\ and \ R_{i_L} > R_j$, for some $j \in \{i_L+1,\ldots,m\}$.

If it happens that $s_j < s_{i_L}$, once again document $i_L$ cannot increase $NDCGL$.

However, if $s_j > s_{i_L}$ for some $j \in \{i_L+1,...,m\}$, it is possible that document $i_L$ is placed below position $i_L$ by $\pi_s$. 

Since only documents $\{i_1,..,i_k\}$ can increase value of $NDCGL$, the maximum $NDCGL$ is when document $i_1$ is exchanged with last document, document $i_2$ is exchanged with 2nd last document and so on.

Thus, we get the following two equations.

\begin{equation*}
\begin{split}
& \quad NDCGL(s,R) \\
= & \frac{1}{Z(R)} \{(G(R_{i_1})- G(R_m))(D(i_1)-D(m)) +\\
& (G(R_{i_2})-G(R_{m-1}))(D(i_2)-D(m-1)) +...+\\
& (G(R_{i_k})-G(R_{m-k+1}))(D(i_k)-D(m-k+1))\}
\end{split}
\end{equation*}

\begin{equation*}
\begin{split}
& \quad \phi^{v^{NDCG}}_{SLAM}(s,R) \\
& = \sum_{i : i \in \{i_1,..,i_k\}} v^{NDCG}_i \underset{j=i+1,...,m}{ \max}\{1.(R_i-R_j) + s_j -s_i) \\
& \ge \sum_{i: i \in \{i_1,..,i_k\}} v^{NDCG}_i \\
& = \frac{1}{Z(R)}\sum_{i:i \in \{i_1,..,i_k\}}(G(R_i)- G(R_m))(D(i)- D(m))
\end{split}
\end{equation*}

It is clear to see from the above two equations that $\phi^{v^{NDCG}}_{SLAM}(s,R) \ge NDCGL(s,R)$, $\forall \ s, \ R$.

\end{proof}

\subsection{Proof of Proposition \ref{gradientboundinperceptron}}

\begin{proof}
For $t \in M$, we have $z_t = X^{\top}_t(\sum_{i=1}^m v^t_ia^t_i)$

1st inequality:
\begin{equation*}
\begin{aligned}
\begin{split}
\|X^{\top}_t(\sum_{i=1}^m v^t_ia^t_i)\|_2 & \le \|X^{\top}_t\|_{1\rightarrow 2}\|\sum v^t_i a^t_i\|_1\\
& \le 2R_X \sum v^t_i \le 2R_X
\end{split}
\end{aligned}
\end{equation*}

2nd inequality:

We should note that $t \in M \implies RML(s_t^{w_t},R_t) \neq 0$. Let $t \in M$. Then, $\exists \ i', k'$ s.t $R_{t,i'} > R_{t,k'}$ but $s^{w_t}_{t,k'}>s^{w_t}_{t,i'}$. 

Now, $\phi^{v^t}(s^{w_t}_t, R_t)= \sum v^t_i c^t_i$. For $(i',k')$, we have $c^t_{i'} \ge 1 +s^{w_t}_{t,k'} - s^{w_t}_{t,i'}>1$.

We should also note that since $R_{t,i'}>R_{t,k'}$, document $i'$ has strictly greater than minimum relevance. Thus, by our calculation of $v^t$ for both ranking measures, we have $v^t_{i'}>0$. Also, by definition, $v^t_{max} \ge 1, \forall \ t.$

Then, $\forall \ i$, $v^t_i \le \ v^t_{max}v^t_{i'} \le\  v^t_{max}v^t_{i'}c^t_{i'}$. Thus, we have
\begin{equation*}
\begin{aligned}
\begin{split}
\sum_{i=1}^m v^t_i \le m v^t_{max}v^t_{i'} c^t_{i'} \le \ & m v^t_{max}(\sum_{i=1}^m v^t_i c^t_i)  \\
& = m v^t_{max} \phi^{v^t}(s^{w_t}_t,R_t)
\end{split}
\end{aligned}
\end{equation*}
Thus, $2R_X \sum v^t_i \le 2 R_X\ m \ v^t_{max} \phi^{v^t}(s^{w_t}_t,R_t)$
\end{proof}

Combining 1st and 2nd inequality, we get $\|z_t\|^2 \le 4 m  R_X^2  m  v^t_{max} \phi^{v^t}(s^{w_t}_t,R_t)$, for $t \in M$.

Since, for $t \notin M$, we have $z_t=0$ and $f_t(w_t)=0$, we get the final inequality

$\|z_t\|^2 \le 4 m  R_X^2 v^t_{max} f_t(w_t)$, $\forall \ t$.

\subsection{Proof of Theorem \ref{theoryboundinperceptron}}

Proof of theorem follows directly by substituting $\|z_t\|^2$ in the regret bound equation of Proposition \ref{regretboundinperceptron}:

\begin{equation*}
\begin{aligned}
\sum_{t=1}^T f_t(w_ t)\ -\sum_{t=1}^T f_t(u) \ \le \ \frac{\|u\|^2}{2\eta} + \sum_{t=1}^T 2\eta R_X^2mv^t_{max} f_t(w_t)\\
\implies (1-  2\eta R_X^2mv^t_{max} ) \sum_{t=1}^T f_t(w_t) \le \sum_{t=1}^T f_t(u) + \frac{\|u\|^2}{2\eta}
\end{aligned}
\end{equation*}

Assuming we can bound $v^t_{max}$ by $v_{max}$, $\forall \ t$, and taking $\eta = \dfrac{1}{4R_X^2mv_{max}}$, we get

\begin{equation*}
\begin{aligned}
\begin{split}
&\sum_{t=1}^T f_t(w_ t)\ \le 2 \ \sum_{t=1}^T f_t(u) + \ \|u\|^2 4 R_X^2mv_{max}\\
\implies & \sum_{t=1}^T RML(s^{w_t}_t,R_t)\ \le 2 \ \sum_{t=1}^T f_t(u) + \ 4 \|u\|^2 m R_X^2 v_{max}
\end{split}
\end{aligned}
\end{equation*}

\subsection{Proof of Corallary \ref{cor:margin}}

Fix a $t$ and the example $(X_t,R_t)$. Set $u = u_\star/\gamma$. For this $u$, we have
\[
\min_{i,j: R_{t,i} > R_{t,j}} u^\top X_{t,i} - u^\top X_{t,j} > 1 ,
\]
which means that
\[
\min_{i,j: R_{t,i} > R_{t,j}} s^u_{t,i} - s^u_{t,j} > 1
\]
This immediately implies that $I(R_{t,i} > R_{t,j})(1 + s^u_{t,j} - s^u_{t,i}) \leq 0$ . Therefore, $f_t(u) = 0$.

\subsection{Proof of Theorem \ref{genbound}}
Our theorem is developed from the expected version of Theorem 6. of \cite{shalev2009}, which is originally given in probabilistic form. The expected version is as follows:

Let $f(w,z)$ be a $\lambda$ strongly convex and $L$-Lipschitz (in $\|\cdot\|_2$) function in $w$. We define $F(w)= E_zf(w,z)$ and $w^* = \underset{w}{\argmin}F(w)$. Let $z_1,..,z_n$ be i.i.d sample and $\hat{w}= \underset{w}{\argmin}\frac{1} {n}\sum_{i=1}^nf(w,z_i)$. Then
\begin{equation}
E[F(\hat{w}) - \ F(w^*)] \le \frac{4L^2}{\lambda n}
\end{equation}
where the expectation is taken over the sample.

The expected version can be observed by carefully going through the proof of Theorem 6. We now derive the expected version of Theorem 7 of \cite{shalev2009}. We start with some definitions. Let $f(w,z)$ be a convex function in $w$. We define $R(w)= E_zf(w,z)$. For i.i.d random sample $z_1,...,z_n$, the population and regularized empirical minimizers are defined as follows
\begin{equation}
\label{optimumw}
w^* = \underset{w}{\argmin}R(w)
\end{equation}
\begin{equation}
\label{empw}
\hat{w}_{\lambda}= \frac{\lambda}{2}\|w\|^2_2 + \underset{w}{\argmin}\frac{1} {n}\sum_{i=1}^nf(w,z_i)
\end{equation}
Expected version of Theorem 7. of \cite{shalev2009} can be stated as follows
\begin{thm}
\label{expectedtheorem}
Let $f : W \times Z \rightarrow \mathbb{R}$ be such that $W$ is bounded by $B$ in $\|\cdot\|_2$, and $f(w, z)$ is convex and L-Lipschitz in $\|\cdot\|_2$  with respect to $w$. Let $z_1,...,z_n$ be an i.i.d. sample and let $\lambda= \sqrt{\frac{\frac{4L^2}{n}}{\frac{B^2}{2}+\frac{4B^2}{n}}}$. Then for $\hat{w}_{\lambda}$ and $w^*$ in Eq.\ref{empw} and Eq.\ref{optimumw} respectively, we have 
\begin{equation}
E[R(\hat{w}) - \ R(w^*)] \le 2LB\left(\frac{8}{n} + \sqrt{\frac{2}{n}}\right)
\end{equation}
\end{thm}
\begin{proof}
Let $r_{\lambda}(w,z)= \frac{\lambda}{2}\|w\|^2_2 + f(w,z)$. Then $r_{\lambda}$ is $\lambda$-strongly convex with Lipschitz constant $\lambda B + L$ in $\|\cdot\|_2$. Applying expected version of Theorem 6 of \cite{shalev2009}, we get

$E(\frac{\lambda}{2}\|\hat{w}_{\lambda}\|^2_2 + R(\hat{w}_{\lambda})) \le \underset{w}{\inf} \ \{\frac{\lambda}{2}\|w\|^2_2 +  R(w) + \frac{4(\lambda B + L)^2}{\lambda n}\}\le \frac{\lambda}{2}\|w^*\|^2_2 +  R(w^*) + \frac{4(\lambda B + L)^2}{\lambda n}$

$\Rightarrow E(R(\hat{w}_{\lambda}) - R(w^*)) \le \frac{\lambda B^2}{2} +\frac{4(\lambda B + L)^2}{\lambda n}$

Minimizing the upper bound w.r.t $\lambda$, we get $\lambda= \sqrt{\frac{4L^2}{n}}\sqrt{\frac{1}{\frac{B^2}{2} + \frac{4B^2}{n}}}$. Plugging it back in the equation and using the relation $\sqrt{a+b} \le \sqrt{a} + \sqrt{b}$, we get Theorem \ref{expectedtheorem}.
\end{proof}

\subsection{Proof of Corollary \ref{sufficient}}

From the definitions preceding Theorem \ref{genbound}, we have $f(w,z)= \ell(s^w,R)$, where $z=(X,R)$ and $s^w=Xw$. Thus, we get the following relation
\begin{align*}
L_2 &= \|\nabla_{w}{\ell(s^w,R)}\|_2 \le \|X^{\top} \nabla_{s^w}{\ell(s^w,R)}\|_2 \\
&\le \|X^{\top}\|_{p \rightarrow 2}\|\nabla_{s^w}{\ell(s^w,R)}\|_p \le R^p_{X} \widetilde{L}_p ,
\end{align*}
where $R^p_{X} \ge \underset {X \in \mathcal{X}}{\sup}\|X^{\top}\|_{p \rightarrow 2}$

Now putting $p=1$, we get
$$\|X^{\top}\|_{1 \rightarrow 2}= \underset{u=1}{\max} \frac{\|X^{\top}u\|_2}{\|u\|_1} .$$
Denoting $X^{\top}_{i}$ as the $i$th column of $X^{\top}$, we have  $\|X^{\top}u\|_2 = \|\sum_{i=1}^m X^{\top}_i u_i\|_2 \le \sum_{i=1}^m |u_i|\|X^{\top}_i\|_2 \le \|u\|_1 \underset{i=1,..,m}{\max} \|X^{\top}_i\|_2$. Since $X^{\top}_i$ is the $d$ dimensional vector representation of a document, assuming bound $R_X$ on $l_2$ norm of each feature vector, we get $\|X^{\top}\|_{1 \rightarrow 2}  \le R_X$.

Thus, we need $\widetilde{L}_1 = \underset{s^w}{\sup}\ \| \nabla_{s^w}{\ell(s^w,R)}\|_1$ to be $m$ independent constant.

\subsection{Calculations for Sec.\ref{applications}}

{\bf RankSVM}:

The RankSVM error is defined as:

\begin{equation*}
\epsilon_{i,j,k}= \ \max(0,I(R_i>R_j)(1+ s_{k,j}- s_{k,i})) 
\end{equation*}
where $k$ indexes query and $(i,j)$ index documents pertaining to that query. The errors are summed up over all pairs of documents and all queries. The relevance vector is binary.

The (sub)-gradient of the loss, w.r.t score vector $s$, is {\bf 0} or ${\bf e_j- e_i}$, depending on $\max$ operator. Thus, the $l_1$ norm of gradient $\le$ 2. Summing over all pairs of documents for a query gives upper bound of $2m$.

{\bf ListNet}:

From Eq.6 of ListNet paper \cite{Cao2007}, we have:
\begin{equation*}
\begin{split}
\dfrac{\partial{L(y^{i},z^{i}_{(s^i)})}}{\partial{s}}= -\sum_{j=1}^m P_{y^i}(x^i_j)\dfrac{\partial{s^i_j}}{\partial{s^i}} + \sum_{j=1}^m \dfrac{exp(s^i_j)}{\sum_{i=1}^m exp(s^i_j)} \dfrac{\partial{s^i_j}}{\partial{s^i}}
\end{split}
\end{equation*}
where $i$ indexes query and $j$ indexes document for that query. Since $\dfrac{\partial{s^i_j}}{\partial{s^i}}= e_j$, we have the $l_1$ norm of the gradient as $\le$:

$\sum_{j=1}^m P_{y^i}(x^i_j)+ \sum_{j=1}^m \dfrac{exp(s^i_j)}{\sum_{j=1}^m exp(s^i_j)} = 2$

{\bf Large margin surrogate}  \cite{chapelle2007}:

The error is:

\begin{equation*}
\epsilon_q= \ \max(0, \underset{y}{\max}(\Delta(y,q) + s_q^{\top}A(y) - s_q^{\top}A(y_q))) 
\end{equation*}
where $q$ indexes query and $y_q$ is the correct ranking corresponding to that query. $A(\cdot)$ is defined in the paper.

The gradient w.r.t $s$ is: $A(y) -A(y_q)$.
 
In worst case scenario, the chosen $y$ will be exact reverse of $y_q$.

Then $A(y) - A(y_q) = [m-1, m-3, m-5, \ldots, 5-m, 3-m, 1-m]^{\top}$.

Thus, the $l_1$ norm of the gradient is $\sim \ 2(1+3+\ldots m) \ \sim O(m^2)$

{\bf  Large margin surrogate}\cite{liu2007}:

The analysis of the $l_1$ norm of the gradient of the loss, w.r.t score vector $s$, is similar to previous analysis. The $m$ independence comes from the fact that the feature map designed by the authors has a normalizing factor (yet guarantees upper bound on the MAP induced loss).

\subsection{Proof of Corollary \ref{cor:k-dependence}}

We first define truncated NDCG
\begin{equation}
\label{eq:NDCG@k}
\begin{split}
NDCG(s,R)@k= \frac{1}{Z_k(R)}\sum_{i=1}^k G(R_i)D(\pi^{-1}_s(i)) 
\end{split}
\end{equation}
where $Z_k(R)= \underset {\pi}{\max}\sum_{i=1}^k G(R_i)D(\pi^{-1}(i))$.  

Like in Sec. \ref{upperbounds}, we have $R_1 \ge R_2 \ge \ldots \ge R_m$, where $R_i$ is the relevance of document $i$. We also note an important property of ranking measures which will be useful in the proof. Ranking measures only depend on the permutation of documents and individual relevance level. They do not depend on the identity of the documents. Thus, documents with same relevance level can be considered to be interchangeable, i.e, relevance levels create equaivalence classes. Thus, without loss of generality, we will assume that $s_i \ge s_j$ if $R_i = R_j$. This is because is $s_i <s_j$, then we can simply interchange the identity of the documents, without affecting the ranking measure.

We define $v^{NDCG@k}$ as  
\begin{equation}
\label{eq:truncatedndcgweights}
v^{NDCG@k}_i = \\
\left\{
	\begin{array}{ll}
		\frac{G(R_i)D(i)}{Z_k(R)}  & \mbox{if } i=1,2,\ldots,k\\
		0  & \mbox{if } i=k+1,\ldots,m .\\ 
           \end{array}
\right.
\end{equation}
We now prove the upper bound property:
Since the documents are sorted according to relevance level, $Z_k(R)= \sum_{i=1}^k G(R)_i D(i)$. Thus $1- NDCG(s,R)@k = \dfrac{\sum_{i=1}^k G(R_i)(D(i)- D(\pi_s^{-1}(i)))}{Z_k(R)}$. 

Now $D(\cdot)$ is a decreasing function. If $ i \ge \pi_s^{-1}(i))$, then the contribution of the $i$th document to NDCG induced loss is non-positive and can be ignored (since $SLAM$ by definition is sum of positive weighted indicator functions). If $ i < \pi_s^{-1}(i))$, that means the document $i$ was outscored by a document with less relevance level. (Keeping in mind that the interchangeability property). Hence the indicator at $i$ would have come on. Since $v^{NDCG@k}_i > \dfrac{G(R_i)(D(i)- D(\pi_s^{-1}(i)))}{Z_k(R)}$, hence we have the upper bound property. 

The proof of Corollary \ref{cor:k-dependence} now follows directly from the proof of Proposition \ref{gradientboundinperceptron}, by noting that in the 2nd inequality, $\sum_{i=1}^m v^t_i \le k v^t_{max}v^t_{i'}c^t_{i'}$. 

\end{document}